\documentclass[11pt]{article}
\usepackage{amssymb, amsthm,amsmath, bm, bbm, mathrsfs, graphicx, color}
\usepackage{mdframed}
\usepackage{setspace}
\usepackage{latexsym}
\usepackage{multicol}

\usepackage{algorithm,algorithmic}

\newtheorem{theorem}{Theorem}

\newtheorem{corollary}{Corollary}

\newtheorem{proposition}{Proposition}

\begin{document}
%\noindent Draft of  6/30/2015
\onehalfspacing
\begin{center}

\textbf{\sc \Large  A  General Hybrid Clustering Technique}

\vspace{0.3cm}

{Saeid Amiri$^{a,}$\footnote{Corresponding author: samiri2@unl.edu},~~Bertrand Clarke$^a$, ~~Jennifer Clarke$^a$ ~and~ Hoyt A. Koepke$^b$}

\vspace{0.3cm}

{\it  $^a$Department of Statistics, University of Nebraska-Lincoln, Lincoln, Nebraska, USA\\
$^b$GraphLab Inc.
Seattle, WA, USA}

\end{center}

\begin{abstract}
Here, we propose a clustering technique for general clustering 
problems including those that have non-convex clusters.  For a given desired number of clusters $K$, 
we use three stages to find a clustering.  The first stage uses a hybrid clustering technique
to produce a series of clusterings of various sizes (randomly selected).   They key steps
are to find a $K$-means clustering using $K_\ell$ clusters 
where $K_\ell \gg K$ and then joins these
small clusters by using single linkage clustering. 
The second stage stabilizes the result of stage one by reclustering via the 
`membership matrix' under Hamming distance to generate a dendrogram.  
The third stage is to cut the dendrogram to get $K^*$ clusters where $K^* \geq K$
and then prune back to $K$ to give a final clustering.  
A variant on our technique also gives a reasonable estimate for $K_T$, the true 
number of clusters.

We provide a series of arguments to justify the steps in the stages of our methods
and we provide numerous examples
involving real and simulated data to compare our technique with other related techniques.

\end{abstract}

\noindent {\bf Keywords}:  hybrid clustering, $K$-means, single linkage, non-convex clusters, stability
%%%%%%%%%%%%%%%%%%%%%%%%%%%%%%%%%%%%%%%%%%%%%%%%%%%
%%
%%         Section 1: Introduction
%%
%%%%%%%%%%%%%%%%%%%%%%%%%%%%%%%%%%%%%%%%%%%%%%%%%%%%

\section{Introduction}
\label{intro}

Clustering is an unsupervised technique used to find underlying structure in a dataset by grouping data points into subsets that are as homogeneous as possible.  Clustering has many applications in a wide range of fields.  No list of references can
be complete, however, three important recent references are  \cite{KaufmanandRousseeuw}, \cite{DudaHartandStork} and \cite{AggarwalandReddy}.   

Arguably, there are three main classes of clustering algorithm: Centroid-based, Hierarchical, and
Partitional.   Centroid-based refers to $K$-means and its variants.
Hierarchical comes in two main forms, divisive
and agglomerative. 
Partitional comes in three main forms graph-theoretic,
spectral, and model-based.  
Because the scope of clustering problems is so big, all of these procedures
have limitations.  
%Centroid-based clustering tends to be limited to clusterings that
%give convex clusters, so non-convex clusters are hard to find accurately.
%Hierarchical clusterings, whether agglomerative or divisive, are based on a
%dis-similarity measure and the goal is to agglomerate or divide sets of data points
%to get clusters with low dis-similarity. It is typically very subjective to decide how much
%dis-similarity is reasonable, especially in the presence of outliers,
%and most of these techniques are one pass meaning
%that once a merge or split has been made, it's hard to reverse.  So, bad decisions can
%get locked in.   Graph-theoretic partitional procedures often work well, but depend heavily 
%on the relevance of the objective function -- e.g., $K$-means can only generate
%convex clusterings and often perform poorly in higher dimensions.   Spectral clustering is 
%often incredibly
%good but also often wildly wrong -- and it's hard to know in advance which case
%will occur especially when there are outliers.  As a practical point, spectral clustering and
%graph theoeretic clustering often seem unduly sensitive to outliers.
%Finally, model-based clustering is
%computationally demanding and only as good as the models used -- knowing
%which models to use is information that is not usually available, especially in
%complex problems.
So, each major class of clustering procedures has its strengths and weakness even if,
in some cases, these are not mapped out very precisely.

The justification for the new method presented here is that it combines two classes
of methods (centroid-based and agglomerative hierarchical) with a careful treatment
of influential data points and (i) is not limited by convexity or (ii) as dependent 
on subjective choices of quantities such as dissimilarities.
That is, we combine several clustering techniques 
and principles in sequence so that one part of the technique may correct 
weaknesses in other parts giving a uniform improvement -- not necesarily decisively
better than other methods in particular cases, but rarely meaningfully outperformed.
In particular, we have not found any examples in which our clustering method
is outperformed to any meaningful extent.  The examples presented here dramatize
this since most of them are very diffucult for any clustering method. 
One consequence of this is that our procedure is well designed for non-convex 
clusterings as well convex ones.  
%Moreover, our procedure gives an unambiguous 
%clustering for each given $K$, the number of clusters desired, subject to 
%choosing a large value $K_\ell \gg K$ where $K_\ell$ is the
%large number of `basal' clusters that are merged to form the desired $K$ clusters.

A further benefit of our clustering method is that we can give formal conditions
ensuring that the clustering will be correct for special cases.  That is, 
we prove a theorem ensuring that the basal sets cover the regions in a clustering
problem, in the limit of large sample size and use this to establish a corollary ensuring 
that the final clustering from our method, at least in simple cases, will be correct.
We also give formal results ensuring that the conditions of our main theorem can be
satisfied in some simple but general cases.   To the best of our knowledge, there are
no techniques, except for $K$-means, for which theoretical results such as ours 
can be established.

To fix notation, we assume $n$ independent and identical (IID) outcomes $x_i,~ i=1,\ldots,n$ 
of a random variable $X$.  The $x_i$'s are assumed $m$-dimensional and written as
$(x_{i1}, \dots , x_{im})$ when needed.   We denote a clustering of size $K$ by 
${\cal{C}}_K = (C_{K1}, \ldots , C_{KK})$; effectively we assume that for each $K$ only one
clustering will be generated.  

For a given $K$, we start by drawing a random
$K_b$, $b = 1, \ldots , B$ from a distribution that ensures a variety of reasonable
clustering sizes will be searched.  Then, the generic steps
are as follows.

\begin{enumerate}

\item {\it Hybrid clustering}: Create $K_\ell$ clusters by $K$-means.  Then use 
single linkage (SL) clustering to take unions of the $K_\ell$ cluster to get clusterings 
of size $K_b$.   

\item {\it Stabilization}: Repeat stage one $B$ times; the result is $B$ clusterings 
with sizes $K_1, \ldots , K_B$.
From these clusterings, form a pooled $n \times \sum_b K_b$ membership matrix $M$.   
%Write $M_1$
%to be the $n \times K_1$ membership matrix for the first clustering 
%in which each row corresponds
%to an $x_i$ and each column represents a cluster in the first clustering.
%The $(s,t)$-entry of $M_1$ is one if $x_s$ is in the $t$-th cluster
%(in the first clustering) and zero otherwise.  Form $M_2$, ..., $M_B$ similarly.
%Then set $M (B) = [M_1, \dots , M_B]$.  
Since each row of $M$ corresponds to an $x_i$
and is a vector of zeros and ones of length $\sum_b K_b$, the Hamming distance $H(x_i, x_j)$
between any two rows can be found and is between one and $B \sum_b K_b$.   
These Hamming distances
give a dissimilarity so that SL clustering can again be applied.

\item {\it Choosing a clustering}:  Use a `grow-and-prune' approach on the dendrogram
from Stage 2.   Cut the dendrogram at some dis-similarity value smaller than $H_K$, the value
of the dis-similarity that gives $K$ clusters.  
%This value, $K^* \geq K$, is meant to represent
%the point at which the chaining property of SL becomes excessive.  
The $K^*$
clusters are then merged to form $K$ clusters after ignoring any clusters that are
too small.

\end{enumerate} 

\noindent
 
 The last stage involves possibly two reclusterings:  One to merge the larger clusters
 down to $K$ clusters and a second to merge the small clusters into these
 $K$ clusters.  The definition of small clusters requires the use 
of a cutoff value $\alpha$, here taken to be 0.05; details are in Sec.  \ref{presentation}.

In Step 1), there is a range of choices for
dis-similarity to be used in the SL clustering.
The usual dis-similarity, namely, defining the minimum distance between
two sets the be the shortest path length connecting them, is one valid
choice.  As seen below, however, it is most effective when the 
data are generated from a probability measure $P$ that has disjoint closed components
each the closure of an open set.
When the components of $P$ are not disjoint, we have found it
advantageous to use a robustified form of the minimum distance between
two sets, namely the 20th percentile of the distances between
the points in the two sets.  This is discussed further in Subsec.
\ref{SLmerging1}

There are precedents for the kind of hybrid clustering described in stages 1) and 2)
that combines two or
more distinct clustering techniques.  Perhaps the closest is \cite{FredandJain} 
Their central idea is to create many clusterings of different sizes (by $K$-means) that can be pooled
via a `co-association matrix' that weights points in each clustering according their
membership.   This matrix can then be modified (take one minus each entry) to give a
dis-similarity so that single-linkage clustering can be used to give a final clustering.  
\cite{FredandJain} refer to this as evidence accumulation clustering
({\sf EAC}) because they are pooling information over a range of clusterings.  
{\sf EAC} differs meaningfully from our technique in three ways.  First, in our technique we
choose a single $K_\ell$ while {\sf EAC} uses a range of cluster sizes.   Second, we ensemble (and hence stabilize)
directly by membership in terms of Hamming distance whereas {\sf EAC} ensembles by a co-association. 
%matrix that includes
%extra weights on membership of points in clusters.  
Third, our procdure uses
an extra step of growing and pruning a dendrogram (see Step 5 in Algoirthm \#1)
that is akin to an optimization over `main' clusters.
%Thus, our technique and {\sf EAC} both use
%$K$-means and single-linkage, i.e., are hybrid techniques, and can find nonconvex clusters,
%but they differ in many of
%the details of the way these techniques are combined.   
Our `fine-tuning' of their technique
seems to give better results.

Another technique that is conceptually similar to ours is due to Chipman and Tibshirani \cite{ChipmanandTibshirani},
hereafter {\sf CT}.
%They describe a hybrid clustering method that ensures certain points are never separated.
First, in a `bottom-up stage', small sets of points that are not to be separated are replaced 
by their centroids.  Then, in a `top-down stage' the remaining points are clustered divisively
to give big clusters. Then, the bottom up and top-down stages are reconciled
to give a final clustering.  Our proposed technique differs from {\sf CT} in four key ways.  First, we use $K$-means in place of {\sf CT}'s `mutual clusters'.  Second, we use single linkage where {\sf CT} uses average linkage. 
Third, our technique has a stabilization stage.
% admitting that boundaries between
%clusters are uncertain whereas {\sf CT} gives results that are constant from run to run. 
Fourth,
our technique uses a `grow and prune' strategy, unlike {\sf CT}.  So,it is 
unclear how well
{\sf CT} performs when the true clusters are non-convex.

A third technique, conceptually related to ours but nevertheless very different, is due to
Karypis et al. (\cite{KarypisHanandKumar}).  This technique, often called
{\sf CHAMELEON}, rests on a graph theoretic analysis of the clustering problem
and uses two passes over the data.  The first is a graph partitioning based
algorithm to divide the data set into a collection of small clusters.  The second pass is an agglomerative
hierarchical clustering based on connectivity (a graph-theoretic concept) to combine these clusters.
%This method is coded in the package {\sf CLUTO}; our technique uses variants on these same
%two stages (splitting the data too finely and then correcting for this).  
%However,
Our method differs from \cite{KarypisHanandKumar}  in four key ways.
First, we use $K$-means instead of graph partitioning.  Second, we simply use single linkage
whereas \cite{KarypisHanandKumar}  combines small clusters based on both closeness and relative
interconnectivity.  Third, our technique has a stabilization stage to manage cluster boundary 
uncertainty.
Fourth, our technique explicitly uses a `grow and prune' strategy permitting a `look ahead' to
more clusters than necessary.
By contrast, {\sf CHAMELEON} has an elaborate optimization.  On the other hand,
both can find
non-convex clusters.

To the best of our knowledge, the earliest explicit proposal for hybrid methods is 
in  \cite{ZhongandGhosh} who observed that using $K$-means with $K$ too
large and single linkage may enable a technique to find nonconvex clusters.  
%However, it appears these authors did not pursue this direction.

In addition to proposing a new hybrid clustering technique (Algorithm \#1) we present 
a way to estimate the correct value $K_T$ of $K$ in Algorithm \#2.   
Essentially, we combine the first three steps of algorithm \#1 with a modification of \cite{FredandJain}.   

%As noted, the first three steps of algorithm \#1 have many
%iterations, say $B$,  based on varying the initial conditions of the $K$-means clustering.   
%These give $B$ membership matrices that are combined into one large matrix that,
%by use of Hamming distance, generates a dissimilarity.  Single linkage with respect to this
%similarity gives a dendrogram to which the `maximum lifetime' method of \cite{FredandJain}
%can be applied -- twice, to find the two largest lifetimes.   
%Cutting off small clusters at these two
%largest lifetimes lets us simply take the mean of the number of clusters these two
%lifetimes represent as our estimate of $K_T$.

%While we do not have any rigorous theory establishing conclusively that these procedures work,
%we provide theoretical motivations and interpretations for the steps in our clustering procedure
%in Sec. \ref{justification}.    We also verify 
%that our clustering technique (with either the 20th percentile
%or the minimum distance similarity) equals or outperforms {\sf EAC }, {\sf CT},
%{\sf CHAMELEON}, and spectral clustering (see \cite{Luxburg} for a particularly lucid
%exposition on spectral clustering)  in numerous
%difficult, nonconvex examples.   Separately, we compare our technique for choosing $K$
%to a variety of other techniques, again verifying that our method typically 
%gives better results than
%several others including the gap statistic, the silhouette distance, the 
%Bayes information criterion (BIC), and {\sf EAC}.  
%Our examples cover convex, non-convex, and high-dimensional data, both real and simulated.

The rest of this paper is organized as follow.  In Sec. \ref{presentation} 
we present our two algorithms for clustering and estimating $K_T$.  
In Sec. \ref{justification} we provide justifications for some of the
steps in our algorithms.  For the steps where we are unable to provide theory, we 
provide methodological interpretations as a motivation for their use. 
In Sec. \ref{comparisons}  we present our numerical comparisons.
Our concluding remarks are
in Sec. \ref{conclusions}.

\section{Presentation of techniques}
\label{presentation}

We begin with Algorithm \#1 that
formalizes our generation of clusterings.   It has five steps and
five inputs:  the number  $K$ of clusters to be in the final clustering,
a number $K_{max}$ to be the largest number of clusters that we would consider
reasonable, a number $B$ of iterations of our initial hybrid clustering technique,
a number $K_\ell \gg K$ of smaller clusters that will be concatenated to
larger clusters, and a value $\alpha$ to serve as a cutoff for the size of a
cluster as measured by the proportion of how many of the $K_\ell$
clusters had to  be combined to create it.  In practice, setting 
$K_\ell = \lfloor n/5 \rfloor $ worked reasonably well; however, $\lfloor n/5 \rfloor$
is an arbitrary choice and we found that adding a layer of variability by choosing $K_\ell$ 
according to a $DUnif[ \lfloor n/4 \rfloor ,  \lfloor n/6 \rfloor]$ gave improved results.
Separately, we also found that larger values
of $K_{max}$ seemed to require larger values of $B$ to get good results.
We address the choice of $B$ and $K_{max}$ later in Sec. \ref{comparisons}.
In our work here, we merely set $\alpha = .05$.  This ensured that
we got at least $K$ clusters in our examples.  Loosely, the more outliers or 
clusters there are, the smaller one should choose $\alpha$.  
So, effectively, given Algorithm \#1, only $K$ must be specified.
The specification of $K$ is done separately in Algorithm \#2.

%We comment that single linkage clustering is commonly criticized because
%of its chaining property, namely that points are merely added to clusters based on
%proximity alone.  By contrast, we regard the chaining property of single linkage
%as a feature to be exploited when putting small clusters -- as opposed to individual 
%points -- together.  That is, we are partially clustering by one method ($K$-means) and
%assembling the results by taking advantage of the chaining property of single linkage.
%This is done in Steps 1 and 5 of Algorithm \#1.   The intermediate steps are
%to ensure that the clustering output by the algorithm is stable.

We begin with our clustering algorithm given in the column to the right.

\begin{center}
\begin{algorithm}
\caption{Stablized Hybrid Clustering  ({\sf SHC}) }
\label{alg03}

\begin{algorithmic}[1]
\begin{enumerate}

\item Given $K$, start by drawing a value of $K_\ell$ and then drawing a value of $K_b \sim DUnif(2,K_{max})$ where
$K_{max} < K_\ell$, for $b=1, \ldots, B$.
For each $K_b$, do the following with randomly generated initial conditions
to obtain  ${\cal{C}}(K_b) = \{C_{b1}, \ldots , C_{b K_b} \}$:
\begin{itemize}

\item  Use standard $K$-means clustering (or any partitional technique) to generate a clustering of size $K_\ell$ `basal' clusters.

\item Next, use single linkage clustering (or any agglomerative technique) to merge the $K_\ell$ basal clusters to get a
clustering ${\cal{C}}_{K_b}$.

\end{itemize}

\item  For ${\cal{C}}(K_1)$, let
$M_1 = (\chi(s, t))_{s=1, \ldots, n; t = 1 , \ldots , K_1}$ be the $n \times K_1$ membership 
matrix with entries 
$
\chi(s,t) = 
\begin{cases}
1 & x_s \in C_{K_1,t} \\
0 &  x_s \notin C_{K_1,t} .
\end{cases}
$
Doing the same for the rest of the ${\cal{C}}(K_b)$'s generates 
membership matrices $M_1, \ldots , M_B$ for clusterings ${\cal{C}}(K_2), \ldots , {\cal{C}}(K_B)$,
respectively.  Concatenating $M_b$'s gives the overall membership matrix 
$M(B) = [M_1, \ldots , M_B]$.

\item From the $n \times \sum_b K_b$ overall membership matrix $M(B)$ we construct a 
dissimilarity matrix using Hamming distance.  Let $S= \sum_b K_b$.
That is, the $i$-th and $j$-th rows in $M(B)$ are of the form 
$x_i = (x_{i,1}, \ldots , x_{i, S})$ and $x_j = (x_{j,1}, \ldots , x_{j, S})$ and so give 
dis-similarities 
$
d_{ij} = d(x_i, x_j) = \sum_{m=1}^{S}  {\sf I}(x_{im}, x_{jm})
$
where ${\sf I}(x_{im}, x_{jm}) = 1$ if $x_{im} \neq x_{jm}$ and zero otherwise.  So, $d_{ij}$
is the number of entries in $x_i$ and $x_j$ that are different
 and $0 \leq d_{ij}\leq S$.
Let $D = (d_{ij} )$ be the resulting matrix.

\item  Given $D$, use SL clustering to generate a vertical dendrogram with leaves
at the bottom and dis-similarity values on the $y$-axis.
Since $K$ is given, it corresponds to a dis-similarity value $H_K$ on the vertical axis,
namely, $H_K$ is the maximum dis-similarity associated with $K$ clusters.
Now, there will be $K$ lines or branches on the dendrogram that cross $H_K$.  
Let the lengths of these lines from $H_K$ down to the next split be denoted $h_1, \ldots , h_K$.
Cut the dendrogram at $H_K+\bar{h}$ and let $K^*$ be the number of clusters 
at that 
%dis-similarity 
value.

\item Write $K^* = K+v$ with $v  \geq 0$.   If $v=0$, the clustering
from Step 4 of size $K$ is the final clustering.  If $v \geq 1$,  
write $v= v_1 + v_2$ where $v_2$ is the number of clusters 
in the clustering ${\cal{C}}_{K^*}$ for which $\#(C_{K^* j} )/n \leq \alpha$.
In the case that $K$ clusters of size at least $\alpha$ do not exist, $\alpha$ is
adjusted downward until $K$ such clusters exist.
Ignore these $v_2$ clusters and using SL (under the corresponding submatrix
of $D$) recluster the points in the remaining $(K+v_1)$ clusters to reduce them
to $K$ `main' clusters. Then, use SL clustering again to assign the points
in the $v_2$ clusters to the $K$ `main' clusters to give the final clustering
of size $K$.
 
\end{enumerate}
\end{algorithmic}
\end{algorithm}
\end{center}
\hspace{1cm}

For brevity, we refer to Algorithm \#1 as {\sf SHC}.

In {\sf SHC}, we have specified the use $K$-means in the first part of Step
1 but left open which dissimilarity to use in Step 2.  This is intentional
because we can establish theory for our method that suggests the usual
minimum distance dissimilarity is best when the components of $P$
are separated (convex or not); however, a dissimilarity between sets
based on 20th percentile of the distance between their points works
better when the separation is not clear or entirely absent.  In
our examples below we denote these dissimilarities
by writing {\sf SHCm} (minimal)  and {\sf SHC20} (20th percentile).

Note that the number of clusters $K^*$ is defined internally to the algorithm in
 Step 4.    The idea is to get a tree that is slightly larger than cutting at $H_K$,
i.e., to let the algorithm search an extra few steps ahead for good clusters.  In Step 
5, any extra clusters that are found 
but not helpful are pruned away.   The intuition behind the choice of $K^*$ is that
the level of the dissimilarity it represents identifies the point at which chaining 
begins to affect the clustering procedure negatively.   

%There are other ways to
%formalize this intuition and the reason to use ours is pragmatic:
%it seems to give good results in practice.

Algorithm \#1 can serve as the basis for another algorithm to estimate $K_T$.
We add an  extra step derived from the method for choosing $K$
in \cite{FredandJain}.  Recall that  \cite{FredandJain}  considered a set of 
`lifetimes' that were lengths in terms of the dissimilarity.  These were the 
distances between the values on the vertical axis at which one could cut a 
dendrogram so as to get a collection of clusters with
the property that at least one of the clusters emerges precisely at the value
on the vertical axis at which the horizontal line was drawn.  
\cite{FredandJain} then cut the dendrogram at the dis-similarity that
corresponds to the maximum of these 
vertical distances to choose the number of clusters.   Algorithm \#2 extends
this method by using it once, removing some clusters, and then using it again.
%In simple cases, our method reduces to the \cite{FredandJain} method.  However, 
%in more complex cases, our method can give either larger or smaller numbers 
%of clusters.

Our general procedure is given in Algorithm \#\ref{alg4}, next page.

\begin{center}
\begin{algorithm}
\caption{Estimate of $K_T$ ({\sf EK})}
\label{alg4}
\begin{algorithmic}[1]
\begin{enumerate}

\item  Use Steps 1-3 from Algorithm \#1 to obtain $D$.

\item Form the dendrogram for the data under $D$ using SL.

\item Use the  \cite{FredandJain} technique to find the two
largest lifetimes.

\item For each of the two  largest lifetimes, cut the dendrogram at that
lifetime and examine the size of the clusters. 
Remove clusters that are both small 
(containing less that $\alpha$100\% of the data) and
split off at or just below $H_K$.    This gives two sub-dendrograms, one
for each lifetime.

\item For each of the sub-dendrograms, cut at $H_K$.  This gives
two numbers of clusters.   Take the mean of these two numbers
of clusters as the estimate of the correct number of clusters.

\end{enumerate}
\end{algorithmic}
\end{algorithm}
\end{center}

For brevity, we refer to Algorithm \#2 as {\sf EK}.

%Clearly, it is possible to run this algorithm many times to take the average
%of the  number of clusters output over the iterations.   There is some evidence 
%that this gives
%slightly better results, but we have not done this here since the improvement
%seemed small compared to the computing time.  We also used two
%dissimilarities here:  One based on the minimum distance and the
%other based on the 20th percentile of the distances (i.e., {\sf EKm} and
%{\sf EK20}).

\section{Justification}
\label{justification}

In this section we provide motivation, interpretation, and properties of the steps in the two algorithms we have proposed.  
%We begin with a theorem in Subsec. \ref{largeKell} 
%demonstrating that the procedure of `using 
%$K$-means with $K$ too large' to find non-convex clusters
%can be expected to give good results when implemented as
%in Subsec. \ref{SLmerging1}.  We also discuss the role of the membership matrix (Step 2-4 of %Alg. 1) and growing and pruning (Step 5  of Alg. 1).

\subsection{$K$-means with large $K_\ell$}
\label{largeKell}

Let the probability measure $P$ have density $p$ and assume
that $p$ only takes values zero and a single, fixed constant.  The places where $p$
assumes a nonzero value are the clusters of $P$.  Our first result
shows that the support of $p$ can be expressed
as a disjoint union of small clusters in the limit of large $n$.  Let
$A \bigtriangleup B$ denote the symmetric difference between sets $A$ and $B$
and for any set $A$,
let ${\sf diam}(A) = \sup_{x,y \in A} d(x, y)$ be the diameter of $A$.  Now, given data 
$x^n = \{ x_1, \ldots, x_n \}$ write 
$\hat{\cal{C}}_K = ( \hat{C}_{K1}, \ldots \hat{C}_{KK} )$
to be a clustering of $x^n$ into $K$ clusters.  
Our result
is the following.

\begin{theorem}
Suppose the following assumptions are satisfied:
\begin{enumerate}

\item $\forall K, m ~ \exists ~C_{Km}$ so that
$$
P( \hat{C}_{Km} \bigtriangleup C_{Km} ) \rightarrow 0
$$
in $P$-probability as $n \rightarrow \infty$.  

\item For any $m$,
$$
\sup_{m=1, \dots , K} {\sf diam}(C_{Km}) \rightarrow 0
$$
as $K \rightarrow \infty$.

\item For each $m$ and $K$, $P(C_{Km}) > 0$.

\item The support of $p$, ${\sf supp}(P)$, consists of finitely
many disjoint open sets with disjoint closures having
smooth boundaries.

\item The random variable $X$ generating $x^n$ is bounded.

\end{enumerate}

Then for any fixed $m$, along any sequence of sets $C_{Km}$ with $P(diam(C_{Km}))>0$, there is a $z \in \overline{\sf Supp}(P)$, the support of $P$, so that
 \begin{eqnarray}
 E(X|\widehat C_{Km}) \longrightarrow z ,
 \end{eqnarray}  
 as $n$ increases first and $K$ increases second, at suitable rates.

\end{theorem}

%{\bf Remark:}  Sufficient conditions for assumptions 1) and 2) to hold
%are given below.   Assumption 3) is a regularity condition ensuring
%all of the clusters in a clustering have non-void interiors.  That 
%is, a clustering of size $K$ necessarily has $K$ nontrivial clusters. 
%Assumption 4) is the strongest of the assumptions because it most limits the
%clusterings to which the theorem can apply.   Assumption 5) can probably
%be relaxed if a more delicate limiting argument is used.

\begin{proof} 
Consider a sequence $\langle \hat{C}_{Km} \rangle \mid_{K=1}^\infty$
for which $P(C_{Km}) > 0$; this is possible by items 1) and 3).  By item 2), $P(C_{Km}) \rightarrow 0^+$.

Step 1:  For such a sequence, 
  \[
E(X|\widehat C_{Km})-E(X| C_{Km})  \overset{P}{\longrightarrow} 0:
  \]

Begin by writing 
   \begin{eqnarray}
&& E(X|\widehat C_{Km})-E(X| C_{Km}) \nonumber \\ 
&=& 
\int_{\widehat{C}_{Km}} \frac{X dP}{P(\widehat{C}_{Km})}
-
\int_{C_{Km}} \frac{X dP}{P(C_{Km})}
\nonumber \\
&=& 
\int_{\widehat{C}_{Km}} \frac{X dP}{P(\widehat{C}_{Km})}
-
\int_{\widehat{C}_{Km}} \frac{X dP}{P(C_{Km})}
\label{term1} \\
&& +
\int_{\widehat{C}_{Km}} \frac{X dP}{P(C_{Km})}
-
\int_{C_{Km}} \frac{X dP}{P(C_{Km})} .
\label{term2}
 \end{eqnarray}
 
 Since $X$ is bounded, term \eqref{term2} goes to zero as
 $n \rightarrow \infty$ by the Dominated Convergence Theorem
 since $I_{C_{Km}}-I_{\widehat C_{Km}}\rightarrow 0$ 
 in $P$-probability under item 1).
 
 To deal with term \eqref{term1}, write it as 
     \begin{eqnarray}
\int_{I_{\widehat C_{Km}}} \left(\frac{1}{P(\widehat C_{Km})}-\frac{1}{P(C_{Km})} \right) X dP .
\label{term3}
\end{eqnarray}

Since $X$ is bounded by $M$, say, the absolute value of
term \eqref{term3} is bounded by

\begin{eqnarray}
M P(\widehat C_{Km}) \left| \frac{1}{P(\widehat{C}_{Km})}-\frac{1}{P(C_{Km})} \right| \nonumber \\
= 
M \left| 1-\frac{P(\widehat C_{Km})}{P(C_{Km})}\right|.
\label{term4}
 \end{eqnarray}
 
 Now, by assumption 1, with probability at least $1 - \eta$, for any $\eta >0$, 
as $n \rightarrow \infty$ we have
       \begin{eqnarray}
\frac{P(\widehat C_{Km})}{P(C_{Km})}\longrightarrow 1.
\nonumber
 \end{eqnarray}
So, the factor in absolute value bars in \eqref{term4}
can be made less than any pre-assigned positive number,
for instance, $\eta/M$, giving that \eqref{term3}
can be made arbitrarily small as $n \rightarrow \infty$.
Consequently,
   \begin{eqnarray}
E(X|\widehat C_{Km})&=&E(X|\widehat C_{Km}) \pm E(X| C_{Km}) \nonumber \\  
&=& o_P(1)+E(X|  C_{Km})\nonumber
 \end{eqnarray}    
 and Step 1 is complete.
 
 Step 2:  By item 2), $\exists z$ such that 
 $C_{Km}\rightarrow \{z\}$.  So, by Step 1, as $n \rightarrow \infty$
 $$
 E(X \mid \hat{C}_{Km} ) \rightarrow z
 $$
 in $P$-probability.  Now, to prove the theorem, it remains to show 
 $z \in \overline{{ \sf supp}(P)}$.
 
 By way of contradiction, suppose $z \notin \overline{{ \sf Supp}(P)}$.
 Then, since $\overline{{ \sf Supp}(P)}$ is a closed set by item 4), its complement is open and hence $\exists \epsilon >0$ so that
 $B(z, \epsilon) \subset \overline{{ \sf Supp}(P)}^c$, where
 $B(z, \epsilon)$ indicates a ball centered at $z$ of
 radius $\epsilon$.  However, consider a sequence of sets $C_{Km}$
 for some fixed $m$ with 
 \begin{eqnarray}
 \forall K:  P(C_{Km})>0;
 \nonumber
 \end{eqnarray}
 such a sequence must exist for some $m$ by Item 3).  By Item 2),
 we have that
 \begin{eqnarray}
{\sf diam} (C_{Km})\rightarrow 0
~ \hbox{as} ~ K \rightarrow \infty.  
 \end{eqnarray}
 So, $\exists K_0$ such that $\forall K \geq K_0$,
 $$
 C_{Km} \subset B(z, \epsilon) ,
 $$
 and therefore $P(C_{Km}) =0$ by letting $n$ and $K$ increase at appropriate
rates, a contradiction.
 Hence, $z \in \overline{{ \sf Supp}(P)}$, establishing the theorem.
\end{proof}

The utility of the theorem stems mostly from the following
corollary.

\begin{corollary}
There exists a $K_0$ so that for $K \geq K_0$, there are 
$m_1,\ldots,m_\ell\leq K$ for sole $\ell$
with 
\begin{eqnarray}
P\left( {\sf Supp}(P) \bigtriangleup \left( \cup_{j=1}^l  \widehat C_{Km_j} \right)^c \right) < \epsilon .
\end{eqnarray}
That is, there are rates at which $n\rightarrow \infty$, $K \rightarrow \infty$ and $\epsilon \rightarrow 0^+$, so that in a limiting sense
\begin{eqnarray}
{\sf Supp}(P) \approx \cup_{j=1}^l \widehat{C}_{Km_j} .
\end{eqnarray}
\end{corollary}

This corollary gives conditions under which the procedure of choosing
$K$ too large, in $K$-means for instance, ensures that the
union of the clusters for that $K$ very closely approximates
the support of $X$, regardless of whether the support is
convex or not.  
%All that matters is that the support of $X$
%consists of disjoint closed components that are the closures of
5disjoint open sets.

Since assumptions 3), 4) and 5) are straightforward to assess, 
we provide sufficient conditions for assumptions 1) and 2)
for the special case of $K$-means clustering.

To do this for assumption 1), recall that
$K$-means uses the Euclidean distance to define the dis-similarity 
$d(x, x^\prime)$ for points $x$ and $x^\prime$.  
Formally, in the limit of large sample sizes, let $\mu_k, ~ k=1,\ldots,K$ be the means of unknown classes $C_{Km}$
under clustering ${\cal{C}}_K = \{ C_{K1}, \ldots , C_{KK} \}$ and let $C$ be the 
membership function 
that assigns data points $x_i$ to clusters i.e., $C(i)=m\Leftrightarrow x_i \in C_{Km}$ under the clustering ${\cal{C}}_K$.  Then the $K$-means
clustering is the ${\cal{C}}_K $ 
that achieves
\begin{eqnarray}
\text{min}_K ~ \text{min}_{\mu_1,\ldots,\mu_K} ~ \sum_{k=1}^{K} \sum_{i:C(i)=k}\| x_i-\mu_k\|^2.
\label{e1}
\end{eqnarray}
(Strictly speaking, the objective function in
\eqref{e1} should be written in its limiting form
\begin{eqnarray}
\sum_{k=1}^K ~ \int_{C_{Km} } \| x - \mu_{Km} \|^2 ~dP(x) 
\label{popKmeans}
\end{eqnarray}
with the constraints $\mu_{Km} = \int_{C_{Km}} X ~ dP$.)

Under the $K$-means optimality criterion, 
given $K$ there are $\mu_1,\ldots,\mu_K \in {\sf Supp}(P)$ such 
that the minimum in \eqref{e1} (or \eqref{popKmeans}) can be written
as 
\[
\sum_{k=1}^{K} \int_{C_{Km}} (x-\mu_{Km})^2 ~P(dx),
\]
with the property that
\begin{eqnarray}
x \in C_{Km} \Longleftrightarrow d(x,\mu_{Km}) \leq d(x,\mu_{K\ell}), \ell\neq k.
\label{convex}
\end{eqnarray}
Defining the centroid of $C_{Km}$ as
\begin{eqnarray}
\mu_{Km}=E\left({x}\mid {x}\in C_{Km}\right)= \int_{C_{Km} } {x} ~P(dx), \nonumber
\end{eqnarray}
with corresponding estimate defined as 
\begin{eqnarray}
\widehat \mu_{Km}=E\left({x}\mid {x}\in \widehat C_{Km}\right)= \int_{{\widehat{C}_{Km}}} {x}~ P(dx) ,
\nonumber
\end{eqnarray}
we can quote the following result.

\begin{theorem}
Under various regularity conditions, as $n \rightarrow \infty$, the 
$K$-means clustering
$\widehat{\cal{C}}_K$ is consistent for ${\cal{C}}_K$.  In particular, 
\begin{eqnarray}
\widehat {\mu}_{Km} \longrightarrow \mu_{Km}, ~ a.s.
\end{eqnarray}
\end{theorem}
\begin{proof}
See Pollard (1981).
\end{proof}

Since $K$-means is a centroid-based clustering, we have that
\begin{eqnarray}
x\in \widehat C_{Km} \Longrightarrow 
\forall \ell \neq m ~ 
d(x,\widehat \mu_{Km}) \leq d(x,\widehat \mu_{K\ell}), \nonumber
\end{eqnarray}
so combining this with \eqref{convex} we get that
 \begin{eqnarray}
 \widehat{\mu}_{Km} \longrightarrow \mu_{Km} \Longleftrightarrow P(\widehat {\cal{C}}_K\bigtriangleup  {\cal{C}}_K ) \longrightarrow 0,
 \end{eqnarray}
 i.e., assumption 1) is satisfied.

Turning to assumption 2), consider the following example with $K$-means
to understand the intuition behind it.  Suppose a data set is generated
as two clusters of the same number of outcomes, one with high 
variance and one with low variance, see the upper left panel
in Fig. \ref{hk3}.   Then, applying $K$-means with increasing $K$,
e.g., $K=2, 4, 10, 14, 18, 20, 24, 30$ in Fig. \ref{hk3}
shows that $K$-means partitions the two clusters more and more finely
but continually assigns more clusters to the high variance data.  
In this context, assumption 2) means that as $n$ increases, the 
clusters will appear to `fill in' yielding $K$
regions with non-void interior for each $K \geq 2$ even if the $n$ required for a given $K$ increases with $K$.

\begin{figure}  \centering
\includegraphics[width=.5\textwidth]{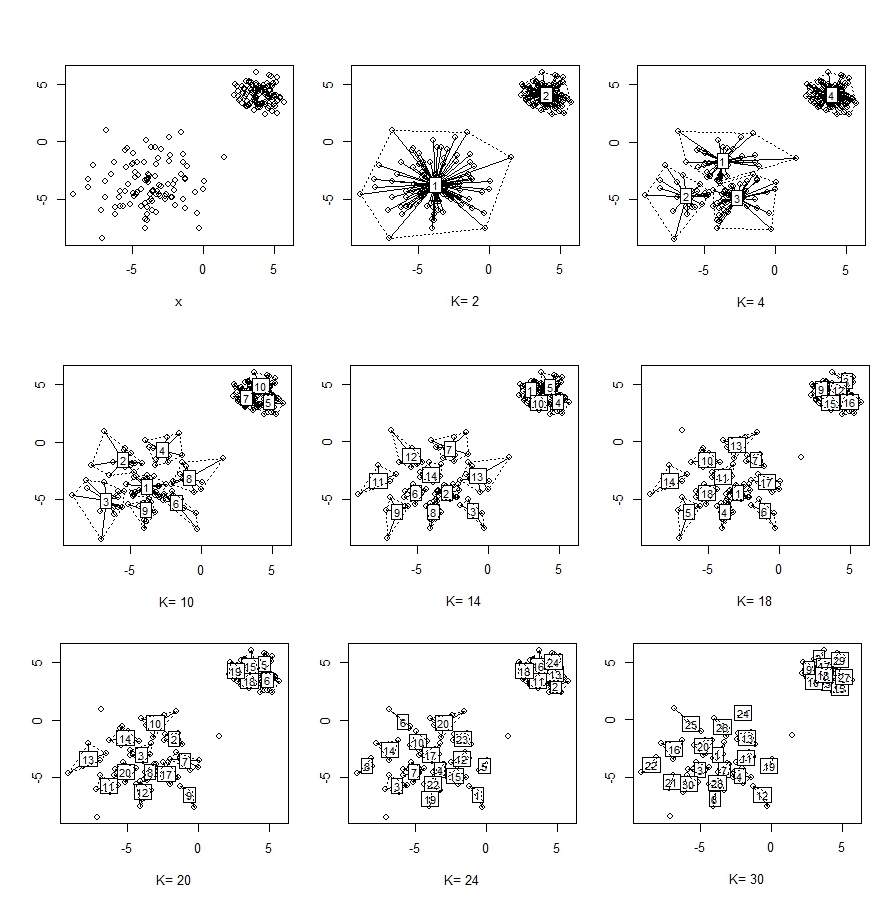} \caption{Plots of two clusters using different 
choices of $K$ to see how clustering divides the true clusters.}
\label{hk3}
\end{figure}

To begin formalizing this intuition, write
\begin{eqnarray}
WSS=\sum_{k=1}^{K} \sum_{i:C(i)=k}\| x_i-\mu_k\|^2,
\end{eqnarray} 
recognizable as the `within clusters' sum of squares.
In the one-dimensional case, suppose $2n$ data points are drawn, 
$X^*_1=(X_{1},\ldots,X_{n}) \sim Unif[0,a]$ and 
$X^*_2 = (X_{n+1}, \ldots , X_{2n} \sim Unif[2a,4a]$.  Clearly, 
$X_1$ represents a low diameter component and $X_2$ represents
a high diameter component.  If we seek a $K$-means clustering
for $K=2$, it is clear that $X^*_1$ and $X^*_2$ should be found.
However, consider $K=3$.  There are two natural clusterings.
The first is to split $X^*_1$ into two clusters of equal size, 
say $C_1$ and $C_2$ letting $C_3 = \{X^*_2 \}$.  The other is
the reverse: Let $C_1 = \{ X^*_1 \}$ and split $X^*_2$
into two clusters of equal size, 
say $C_2$ and $C_3$.  It is easy to verify that the population
value of WSS for the first clustering is $(9/24)a^2$
while for the second clustering it is $(3/24)a^2$.  This means
the second clustering, splitting the high diameter component,
gives a smaller WSS.  Since $K$-means chooses the mean of WSS
over the number of clusters, in this case, $K$-means would choose
the second clustering.

The example can be continued for higher $K$, higher $K_T$,
and other distributions continuing to show that $K$-means tends
to split the largest cluster until it is worthwhile to
split the smaller cluster and then resumes splitting the larger
cluster, and so on.  A consequence of this is that
$\hat{C}_{Km}$ tends to decrease in size as $K$ increases and
this suggests that $C_{Km}$ will similarly decrease as assumed
in Item 2) as $n \rightarrow \infty$.  We state a version of this 
in the following.

\begin{proposition}
Suppose a clustering method continually splits the largest cluster
on the population level
as $K$ increases.  Then,
given $\delta > 0$, there is a $K_0$ so that
\begin{eqnarray}
k>K_0 \Longrightarrow {\sf diam}(C_k)<\delta. 
\end{eqnarray}
\end{proposition}

\begin{proof}
Let $\{ \mu_{K1},\ldots,\mu_{KK} \}$ be the centers of 
the optimal clusters and write
\[
C_{Km}=\{x: 
|x- \mu_{Km}|<|x-\mu_{K \ell} | \}
\]
for $\ell \neq m$.  Let 
\[
\delta_0=\max\{{\sf diam}(C_{Km}) : m = 1, \ldots , K \}.
\]
Then, for $K^\prime$ large enough,  
\[
\max\{  {\sf diam}(\widehat{C}_{K^\prime m}) : 
m = 1, \ldots , K^\prime \} \leq \frac{\delta_0}{2} .
\]
Since this process can be repeated the proposition is established. 
\end{proof}

Taken together, the results of this subsection justify the 
$K$-means part of Step 1) of Algorithm \#1.

\subsection{Merging the 'basal' clusters}
\label{SLmerging1}

Next we turn to justifying the use of single linkage (SL) in the second
part of Step 1 in Algorithm \# 1.  Recall, SL means that we merge
sets that are closest i.e., given a distance
$d$ on, say, $C_{K1}, \dots , C_{KK}$, SL clustering merges the two sets 
that achieve 
$$
\min_{m, m^\prime = 1, \dots , K; m \neq m^\prime}  d(C_{Km}, C_{Km^\prime} ) .
$$
The question that remains is how to choose $d$.   Here we use two
choices.   The first is to write $d$ as 
\begin{eqnarray}
d_{usual}(C_{Km}, C_{Km^\prime}) = \min_{x \in C_{Km}, x^\prime \in C_{Km^\prime }} d(x, x^\prime) ,
\label{usualdist}
\end{eqnarray}
where $d$ is a metric e.g., Euclidean,
i.e., $d_{usual}$ gives the distance between two sets as the minimum over the distances
between their points.

However, being an order statistic, \eqref{usualdist} can be affected by
extreme values in the data set.  
So, we stabilize 
$d_{usual}(C_{Km}, C_{Km^\prime})$ by replacing it with the 20th 
percentile of the distances between
points in $C_{Km}$ and $C_{Km^\prime}$.
That is, for $m \neq m^\prime$, we find the distances
$$
\{
d(x, x^\prime) :
x \in C_{Km}, x^\prime \in C_{Km^\prime} \},
$$
take their order statistics, and find the approximately
$\lfloor .2 \#(C_{Km}) \#(C_{Km^\prime}) \rfloor$ order
statistic.  (Finding a non-integer order statistic is done 
internally to the R program using linear interpolation.)
We call the resulting dissimilarity $d_{20}$, i.e.,
\begin{eqnarray}
d_{20}(C_{Km}, C_{Km^\prime}) = 20-\mbox{th percentile of} ~ 
\nonumber \\
 \{ d(x, x^\prime) : {x \in C_{Km}, x^\prime \in C_{Km^\prime }} \}
\label{20percent}
\end{eqnarray}
to indicate it
is based on the 20th percentile of the distances between points
in the two sets.  Thus, with $d_{20}$
we are using single linkage with respect to a 
dissimilarity that should be robust against extreme values.
Other percentiles such as the fifth or tenth can also be used, but they
gave values between $d_{usual}$ and $d_{20}$
in the examples we studied.  It seemed from
our work that $d_{20}$ gave the best resuts in cases where $d_{usual}$ did not.

For the sake of completeness we next give conditions under 
which $d_{usual}$ can be expected to perform well.  We are
unable to demonstrate this for $d_{20}$ but suggest there will
be an analogous result since using $d_{20}$ gave results that
were essentially never worse (and sometimes better) than $d_{usual}$.

Suppose $\overline{\sf Supp}(P)$ consists of $K_T$ disjoint
regions each being the closure of an open set, assumed disjoint from
the other open sets.  Let $\delta$
be the minimum distance between points in disjoint components,
i.e.,
$$
\delta = \min_{m, m^\prime = 1, \ldots , K_T; m \neq m^\prime}
\min_{x \in C_m, x^\prime \in C_{m^\prime}} d(x, x^\prime).
$$
If $n$ and $K$ are chosen so large that all the $\hat{C}_{Km}$'s 
for $m= 1, \ldots , K$ have ${\sf diam}(\hat{C}_{Km}) < \delta/3$
(any number strictly less than $\delta$ will suffice), choose a
regular grid $G$ of points in $\overline{\sf Supp}(P)$ so that the
distance between two adjacent points on the same axis
is less than $(1/2)(\delta/3)$.  This ensures that each $\hat{C}_{Km}$
has at least one grid point in it.  
The points in $G$ are essentially a perfectly representative
set of $\overline{\sf Supp}(P)$ and hence of $P$. 
Now, if we apply SL with
$d_{usual}$ to the points in $G$ we will always put points
or subsets in the same component together before we merge points
or subsets of any two distinct components.  That is, the metric 
on $G$ ensures that the closest point to any other point will always
be in the same component if possible.  So, we have proved the
following theorem.

\begin{theorem}
If the components of $\overline{\sf Supp}(P)$ are disjoint
there is a cut point in the dendrogram of the SL merging under
$d_{usual}$ of the
points in $G$ that separates the components perfectly.
\end{theorem}

Now, if $n$ is large enough, the data set can be taken as 
perfectly representative of $\overline{\sf Supp}(P)$ and hence of $P$,
i.e., it is a good approximation to $G$ in the sense of
filling out all the components of $P$.  Hence, it follows that
SL using $d_{usual}$ can perfectly separate the components
of $\overline{\sf Supp}(P)$, and hence of $P$, in a limiting sense.
This does not require convexity of the components of $P$, only that the data points can be
regarded as essentially a perfect representation of $P$.

Note that one of the key hypotheses of this theorem forces the components
of $P$ to be separated.  In fact, this is often not the case -- components
may touch each other at individual points or may be linked by a very thin short line.  In these cases,
the components of
${\sf Supp}(P)$ may not have disjoint closures or the closure of the components may not give 
$\overline{\sf Supp}(P)$, respectively.  When assumption 4 of Theorem 1 is satisfied,
we have found that $d_{usual}$ works well:  In a limiting sense, two basal sets from the
same component will always be joined before either is joined to another component.
However, when hypothesis 4 of Theorem 1 is not satisfied, $d_{usual}$ does not have this property.
In these cases, we have found $d_{20}$ to work better; this is seen in Subsec. \ref{aggregation}.
It should be noted that the examples in Subsecs. \ref{halfring} and \ref{flamedata} also do not
seem to satisfy hypothesis 4 but for these cases $d_{usual}$ and $d_{20}$ give
comparable results.    We regard this as a reflection of the fact that hypothesis 4 is necessary but not 
sufficient for the conclusion of Theorem 1.  Also, although not shown here,
we examined our clustering technique using $d_5$ and $d_{20}$ in the
sense of \eqref{20percent} but found they were outperformed by at
least one of $d_{usual}$ or $d_{20}$.   One point in favor of $d_{usual}$ is that it is interpretable
in that two points are in the same cluster merely if they are close enough, unlike
$d_{20}$. 

Why does the 20-th percentile work well in cases where hypothesis 4 is not satisfied?
While we do not have a formal argument, the intuition may be expressed as follows.
If hypothesis 4 is not satisfied and the clusters are highly non-convex 
$d_{usual}$ will be much more sensitive to the boundary values of the clusters than
$d_{20}$.  Consequently, there may be overly influential data points -- data points that 
are valid but far from other data points -- that will affect the sequence of merges of the basal
clusters in ways that are not representative of the support of $P$.  Using $d_{20}$ in place
of $d_{usual}$ reduces the influence of these data points eliminating distortions of the 
path by which 
basal clusters are merged.  We do not have a rule for when to use $d_{usual}$ versus $d_{20}$,
however, the presence of extreme points (as opposed to outliers) is a good indicator that $d_{20}$ should be
preferred and this is consistent with all our examples.  Indeed, the examples
in Subsecs. \ref{halfring} and \ref{flamedata}, where $d_{usual}$ and $d_{20}$ give
equivalent results, do not have clusters with extreme points.

\subsection{Using the overall membership matrix}
\label{membership}

In Steps 2 and 3 of algorithm \#1, a composite membership matrix $M(B)$
for $B$ clusterings is defined.  Then, single linkage clustering is 
applied to the rows of $M(B)$ in Step 4.  Because $D$ is absed on $M(B)$ our results should be
robust results because by using several random starts
for the clustering and looking only at which
cluster a data point is in, we are ensuring that
the final clustering is a sort of `consensus clustering'
representing what is invariant under two sorts of randomness -- randomness of the clustering and random noise in the data points themselves.

Our use of the matrix $M(B)$ means our method may be regarded as an ensemble approach.  Each set of columns in
$M(B)$ represents a clustering and pooling over clusterings in step 4
effectively means that we are analyzing $B$
different clustering structures for the data.
The analog of ensembling the matrices is played by 
single linkage which groups similar clusterings together.
The result is that the final clustering is stabilized.

\subsection{Growing and pruning}

In Step 4, algorithm \#1 grows a dendrogram
of size $K^*$ by single linkage.   In fact, $K^*$ may be
bigger than the size $K$ of the desired clustering.  
In such cases, the dendrogram `grown' is too large
and must be pruned back.  This is done in Step 5.

The benefit of this is that by growing a dendrogram a
little larger than required, the method may look one or
more splits further along so that outliers or other aberrant points may 
be removed.  The outliers or other aberrant points are
in the $v_2$ small clusters that are removed before
the data are reclustered.  Leaving out the $v_2$ small
clusters means that the resulting clustering should be
more stable, and therefore hopefully more accurate.  Of course, the
outliers  and aberrant points in the $v_2$ small clusters
must be merged back into the clustering as is
done at the end of Step 5.  However,  they are merged back
into clusters they were not used to form.  Hence, the 
final clustering may
be more representative of $P$ than if the extra points
were used to form the clusters in the first place.

At root, Steps 2-5 are designed to take advantage of the 
chaining property of single linkage.  Usually, the chaining property is a reason not to use single linkage; here the chaining property is used only to 
fill out clusters but as far as possible not to merge
them.   In terms of filling out clusters, the chaining property is desirable.  It only becomes disadvantageous
when it inappropriately
joins clusters.

\subsection{Estimating $K_T$ by using lifetimes}

Steps 1 and 2 in Algorithm 2 have been addressed in Subsecs. \ref{largeKell}, \ref{SLmerging1}, and 
\ref{membership}.  So, it remains to justify the use 
of lifetimes in Steps 3-5 for estimating $K_T$.

As can be seen in 
Fig. 3 of \cite{FredandJain} where they give an example of lifetimes for a dendrogram,
defining clusters by the use of a maximum lifetime has the tendency to amplify the separation between so that points are usually only put in their final cluster near the leaves
of the dendrogram.
That is, there are often several long lifetimes that give reasonable places to
cut the dendrogram such that the clusters at the bottom are well separated and homogeneous
in the sense that further decreases in dissimilarity are small.   This method seems to work well when
the clusters are well separated, regardless of whether they are convex.  

Our refinement of
the \cite{FredandJain} technique is an effort to extend it to cases where the separation
among clusters is not as clear.   Indeed, removing subsets of data that are too small before applying 
\cite{FredandJain} ensures that likely outliers or
other aberrant points will not affect the collection of lifetimes.  
The benefit is that outliers and aberrant points will rarely be seen as separate clusters,  
yielding a more accurate number of clusters.

\section{Evaluation of the our techniques} 
\label{comparisons}

In this section, we compare the performance of the two proposed techniques
with the existing techniques described  in Sec. \ref{intro} using 
eight data sets that are qualitatively different
from each other.  The first five are  the {\sf 3-NORMALS}, 
{\sf AGGREGATION}, {\sf SPIRAL},  {\sf HALF-RING}, {\sf FLAME} data sets found at \cite{Webref}.   These two dimensional data sets are `shape data'.
{\sf 3-NORMALS} is simulated from three normal distributions giving convex shapes that are not well separated.
{\sf AGGREGATION} has one cluster that is non-convex and several others that
are convex but not separated.   
{\sf SPIRAL} has three separated but nonconvex and `intertwined' clusters.  {\sf HALF-RING}
has two separated clusters that are non-convex with different densities making it ambiguous whether one of the clusters
should be split or not.  {\sf FLAME} has two clusters, one convex, the other non-convex.
The two clusters are not well-separated and the convex cluster has some outliers.  
We did not examine other shape sets at \cite{Webref} because they were
similar to data sets we had used or were too challenging for all methods.   

For four of these five data sets, we considered seven different clustering techniques: 
 $K$-means ({\sf K-m}), EAC, CT, {\sf CHAMELEON} (hereafter abbreviated to {\sf CHA}, 
spectral clustering {\sf SPECC},  {\sf SHCm}, and {\sf SHC20}.   We applied all seven to 
{\sf AGGREGATION},  {\sf SPIRAL},  {\sf HALF-RING}, and {\sf FLAME}
but only applied {\sf CHAMELEON} to one output from {\sf 3-NORMALS} because
{\sf CHAMELEON} is intended for nonconvex data and is cumbersome when doing
many repetitions with simulated data.

The first five data sets are two dimensional, so it is enough to compare the output of the clustering
techniques visually.   However, because we can unambiguously assign a `true' clustering to these
data sets, we can also calculate  an
accuracy index, AI, i.e.,  the proportion of data points correctly assigned to their cluster.
This was calculated using the software described in \cite{Fraleyetal2012}.    Since there is
randomness built into {\sf K-m}, {\sf EAC} , {SPECC}, and our method, where necessary i.e., for
non-synthetic data, we repeated the techniques
and report the mean AI (MAI) and its standard deviation (SAI).  This was not necessary for CT since
it does not vary over repetitions (hence SAI  for CT is always zero).   
The results are in Subsecs. \ref{3_normals} and \ref{nonconvex}.

The last three data sets have four or more dimensions.  The first of these is Fisher's
familiar {\sf IRIS} data that has four dimensions.   The second is 
the  {\sf GARBER} microarray data set
found at  \cite{Webref3}.  It has 916 dimensions. The third is the
 {\sf WINE QUALITY} data set found  at \cite{Webref2}.  This data set has 11 variables and is really
two data sets, one for red wine and one for white wine.  
We used all seven techniques on {\sf IRIS} and {\sf GARBER} but dropped {\sf CHAMELEON} for
{\sf WINE QUALITY}  because it was hard to implement and, being graph-theoretic, it cannot be 
expected to perform well on data that have many dimensions.   We also dropped {\sf CT}
for {\sf WINE QUALITY} since {\sf CT} so rarely performed well.
In these examples, the data sets are from classification problems so we have assumed that
a unique true clustering exists as defined by the classes.   We can again calculate the MAI and SAI as described above.
The results are in Subsec. \ref{higherdims}.

In all eight examples, we set $B=200$ for {\sf EAC}, {\sf SHCm}, and {\sf SCH20} to ensure fairness.  
In the first seven examples 
we set $K_{max}=25$ in {\sf SHC}; in the {\sf GARBER} data set we chose $K_{max} = 11$ because
$K_{max} < K_\ell-2$ and $\lfloor n/5 \rfloor =  \lfloor 74/5 \rfloor = 14$
so that it made sense to draw $K_\ell$'s from $DU(\lfloor n/6 \rfloor,\lfloor n/4 \rfloor)=DU(12,18)$.   The implication of our examples is that 
while other methods may equal or even perform slightly better than one or both of the
{\sf SHC} methods in some cases, no competitor beats them consistently by a 
nontrivial amount.

We begin with straightforward comparisons of the clustering techniques for fixed $K$ and then turn to
comparing the techniques for choosing $K$.  Specifically, for all eight data sets, we compare six techniques 
for choosing $K$, namely, the silhouette distance, the gap statistic, BIC, 
EAC, and two methods based on Algorithm \# 2 ({\sf EKm} and {\sf EK20}).  These results are
given in Subsec. \ref{estK}.

\subsection{Convex example: {\sf 3-NORMALS}}
\label{3_normals}

In order to study this convex clustering problem, consider Figure \ref{cov1} showing $n=120$ 
observations that clearly form three clusters. 
The data were generated by taking 40 independent and identical draws from each of three normal
distributions.  Specifically,  the there normals are  $(X_i, Y_i)^T  \sim N(\mu_j,\Sigma_j)$, $j=1,2,3$ where $$\mu_1=\left(
    \begin{array}{c}
      2 \\
      2
    \end{array}
  \right) , ~
\Sigma_1=\left(
    \begin{array}{cc}
      .7&0 \\
      0&.7
    \end{array}
  \right) ,
$$
$$
\mu_2=\left(
    \begin{array}{c}
      -2 \\
      2
    \end{array}
  \right),  ~ \Sigma_2=\left(
    \begin{array}{cc}
      .7&0 \\
      0&.7
    \end{array}
  \right),
$$ 
and
$$
\mu_3=\left(
    \begin{array}{c}
      0 \\
      -1
    \end{array}
  \right) , ~ \Sigma_3=\left(
    \begin{array}{cc}
      1.5&0 \\
      0&0.4
    \end{array}
  \right) .
$$  

Applying the seven clustering techniques to one set of the  {\sf 3-NORMALS} data with $K_T=3$
gives Fig. \ref{cov1}.  First, the upper left panel shows the true clusters.
It appears that {\sf K-m}, {\sf CT}, {\sf SPECC}, and {\sf SHC20}  do roughly equally well,
although spectral clustering tends to enlarge the right hand cluster unduly.
By contrast, {\sf CHA}, {\sf EAC}, and {\sf SHCm} do not give intuitively reasonable results.  
{\sf CHA} makes the
lower left cluster too small; {\sf EAC} and  {\sf SHCm} effectively merge the right and bottom clusters.

\begin{figure} [htp!]
\centering
\includegraphics[width=.5\textwidth]{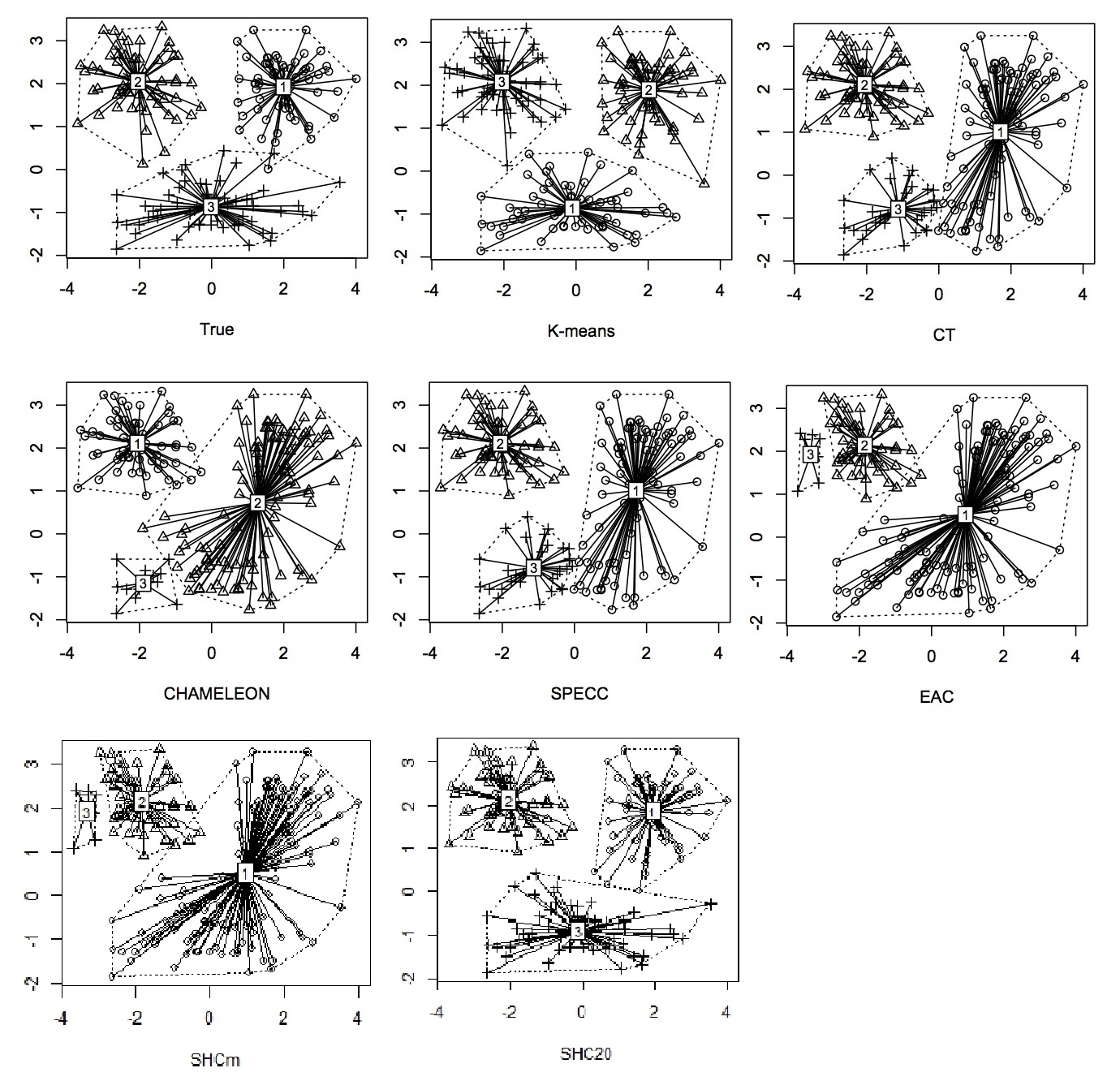} \caption{The true clustering of one run of {\sf 3-NORMALS}
data and seven estimated clusterings using  {\sf K-m}, {\sf CT}, {\sf CHA}, {\sf SPECC}, 
{\sf EAC}, {\sf SHCm}, and {\sf SHC20}. }
\label{cov1}
\end{figure}

These observations are mostly but not thoroughly consistent with Table \ref{tabcon}
in which MAI values are rounded to two decimal places and SAI values are rounded
to three decimal places.
Table \ref{tabcon} is the summary of the performance of
the six methods over 200 simulated data sets; {\sf CHA} is omitted as noted earlier.   
Clearly,  despite Fig. \ref{cov1}, {\sf CT} and {\sf EAC} are poor in an average sense (MAI) while
{\sf SCHm} is comparable to {\sf SHC20} -- suggesting the single example of {\sf SHCm}
 in Fig. \ref{cov1}
is unrepresentative of its general performance.
The performance of spectral clustering is better than Fig. \ref{cov1} suggests
but not as good as $K$-means which is best as expected.   
Loosely, {\sf K-m}, {\sf SPECC}. and the two {\sf SHC} techniques
seem to be broadly equivalent.  Note that, usually, the size of the SAI values are
higher for poorer performing clustering techniques.  This can be taken as an 
assessment of the quality of the clustering.

\begin{table}[h]
\centering
   \caption{The comparison of the proposed methods on {\sf 3-NORMALS} data}
\begin{tabular}{lcccccc} \hline
&\multicolumn{6}{c}{method}
\\ \cline{2-7} 
          & K-m  &   CT        &   SPECC     &   EAC     &   SHCm &SHC20\\ \hline
MAI   &.97         &   0.87    &    0.94      &   0.84    &0.92 &0.93  \\
SAI    &0.000       &   0.031  &  0.114     &  0.158    & 0.126 & 0.118\\
\hline
\end{tabular}\\  \label{tabcon} 
\end{table}

\subsection{Non-convex examples}
\label{nonconvex}

Here we compare all seven clustering techniques for the {\sf AGGREGATION}, {\sf SPIRAL},  {\sf HALF-RING}, and
{\sf FLAME} data sets.

\subsubsection{ {\sf AGGREGATION} data}
\label{aggregation}

The {\sf AGGREGATION} data, depicted in the top panel of Figure \ref{aggr}, is used in  \cite{Gionisetal} to show 
the performance of ensembling.   

If $K_T=7$ is used, the clusterings from the best three methods ({\sf CHAMELEON}, {\sf EAC}, and  {\sf SHC20})
are shown in the lower panels of Fig. \ref{aggr}.   It is seen that {\sf EAC} is a little worse than {\sf CHAMELEON},
and {\sf SHC20} because it merges clusters to form cluster 2 and divides a cluster to give clusters 5 and 6.  
{\sf CHAMELEON}, being graph theoretic, is better at separating the two clusters while
{\sf SHC20} includes randomness and therefore separates the two clusters slightly differently over different runs.
{\sf EAC} also includes randomness so the panels in Fig. \ref{aggr} merely show one run.
 The result of {\sf SHCm} is also shown and is noticably worse:  {\sf SHCm} merges two 
clusters inappropriately
and splits another cluster inappropriately into three clusters.  This is the only case
among the examples here in which {\sf SHCm} and {\sf SCH20} give 
meaningfully different results, apparently because of the
extreme points in the upper left cluster; note that assumption 4) in
Theorem 1 is violated.

\begin{figure} \centering
\includegraphics[width=.52\textwidth]{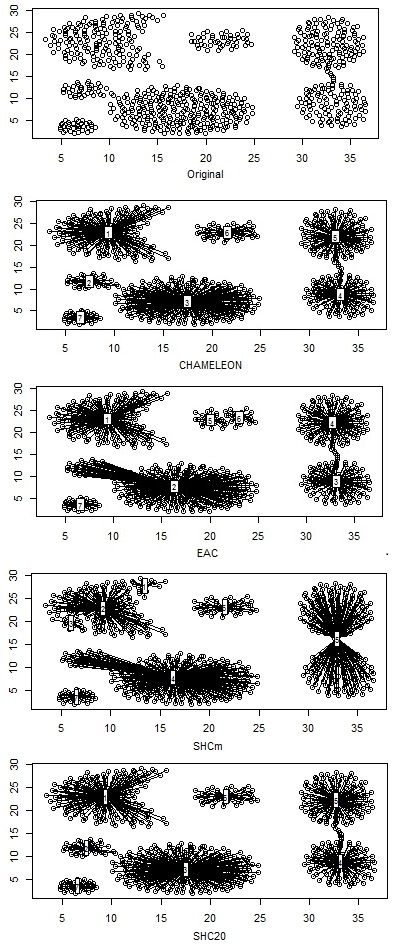} \caption{Top panel:  Original {\sf AGGREGATION}
data.  Second panel:  Clustering under {\sf CHAMELEON}.  Third panel: Clustering under {\sf EAC}.
Fourth and fifth panels: Clustering under {\sf SHCm} and {\sf SHC20} respectively}.  
\label{aggr}
\end{figure}
   
These general appearances are consistent with the results in Table \ref{tabagg}.  Note that 
{\sf CT} and {\sf CHAMELEON} have SAI zero because they are one-pass methods.  The overall inference is
that {\sf SHC20} and {\sf CHAMELEON} are essentially equivalent in this example.
   
   \begin{table}[h]
   \centering \scriptsize
      \caption{The comparison of the proposed methods on {\sf AGGREGATION} data}    
  \begin{tabular}{lccccccc} \hline
  &\multicolumn{7}{c}{method}
  \\ \cline{2-8} 
            & K-m  &   CT        &   SPECC    & {\sf CHA}&   EAC     &   SHCm&SHC20\\ \hline
  MAI   &     0.83  &  0.81     &  0.92   & 1      &0.95 &0.84  & 0.98 \\
  SAI    &     0.000  & 0      & 0.044  &0    &   0.000&0.000  & 0.044\\
  \hline
  \end{tabular}\\  \label{tabagg} 
  \end{table}

\subsubsection{ {\sf SPIRAL} data}

Figure \ref{spiral} shows the {\sf SPIRAL} data that are often considered as a test case for nonconvex clustering.
Clearly, $K_T=3$ and the clusters are the three lines of points.  In this run, only {\sf EAC}, {\sf SHCm} 
and {\sf SHC20}
find the correct clusters.  More generally, if one uses many runs, the  MAI and SAI present a slightly different
result:  The best methods are {\sf SPECC} and the two {\sf SHC}'s.
Of these, {\sf SHC20} and {\sf SHCm} should be preferred because {\sf SPECC} has a 
higher SAI, as can be seen in Table \ref{tabspiral}.
    
  \begin{figure} [htp!]
 \centering
  \includegraphics[width=.5\textwidth]{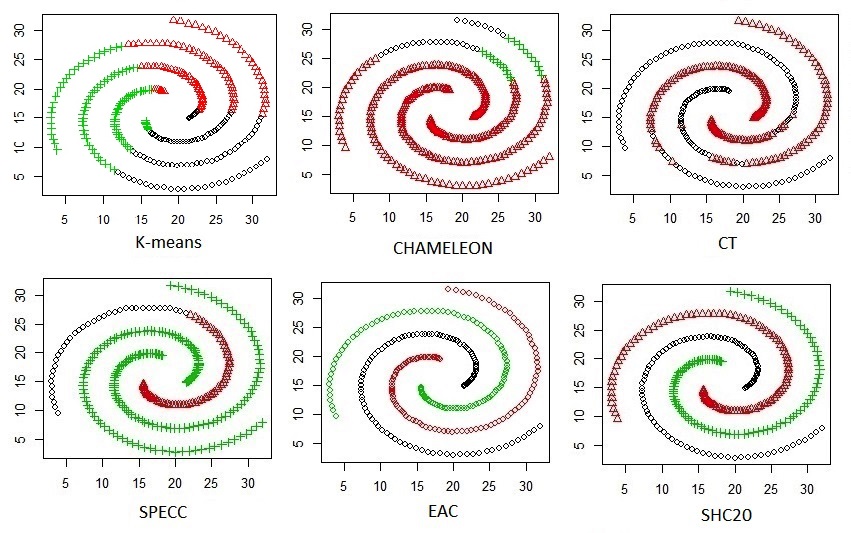} \caption{Clustering of the {\sf SPIRAL}  data with six 
different methods as indicated on the panels; SHCm and SCH20 are idnetical
so only SHC20 is shown.}
  \label{spiral}
  \end{figure}

\begin{table}[h]
\centering \scriptsize
   \caption{The comparison of the proposed methods on {\sf SPIRAL} data}
\begin{tabular}{lccccccc} \hline
&\multicolumn{7}{c}{method}
\\ \cline{2-8} 
          & K-m  &   CT   &  {\sf CHA}   &   SPECC     &   EAC     &   SHCm&   SHC20\\ \hline
MAI   &   0.34      &   0.35&0.44    &  0.94        &   0.90    & 1&1  \\
SAI    & 0.000      &     0&0&0.148        & 0.134    & 0.000&0.000 \\
\hline
\end{tabular}\\  \label{tabspiral} 
\end{table}

\subsubsection{ {\sf HALF-RING} data}
\label{halfring}

Figure \ref{bowl} shows six clusterings of the {\sf HALF-RING} data, considered in \cite{JainandLaw} ({\sf SHCm} and {\sf SHC20} are nearly identical).  
Intuitively, there are two clusters but the density of the points makes it ambiguous whether the top 
half-ring should be split into two clusters or not.   It can be seen that {\sf K-m} and {\sf CT} give poor performance (they merge the left half of the bottom half ring to the top half ring)
but the other methods find the two clusters;
this is seen in the results in Table \ref{tabhalf}.
Note that when {\sf SHCm} and {\sf SHC20} are essentialy the same, we only show one of them.

While {\sf SHCm} and {\sf SHC20} have slightly lower MAI's than the other three good methods ({\sf CHAMELEON}, {\sf SPECC} and {\sf EAC}) they
also have nonzero SAI's indicating the ambiguity in the top half-ring.  
Indeed, when we ran {\sf SHC20} 1000 times
on the {\sf HALF-RING} data, the two half rings were in separate clusters 873 times while the right hand portion of the
top half-ring was put in the same cluster as the bottom ring 127 times.  The ambiguity of two versus three clusters
being appropriate is recognized by the {\sf SHC}'s whereas the SAI's of the other 
three good methods 
({\sf CHAMELEON}, {\sf SPECC} and {\sf EAC}) being zero indicates they are not recognizing the ambiguity. 
% Indeed, the difference
%between the {\sf HALF-RING} data having decisively two or three clusters depends on a very %small number of points.
%For instance, if $K_T$ is set to be three rather than two, only {\sf EAC}, {\sf SHCm}, and {\sf %SHC20}  reliably find three clusters.

    \begin{figure}  \centering
    \includegraphics[width=.5\textwidth]{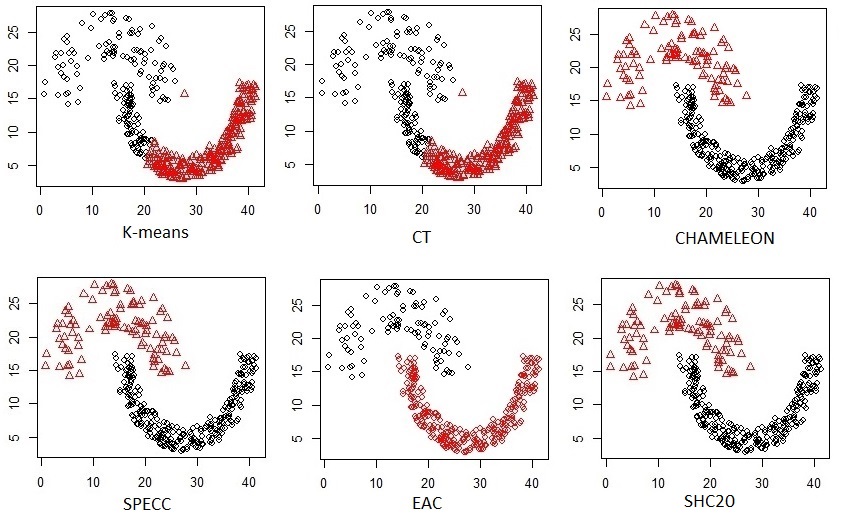} \caption{Clustering of the  {\sf HALF-RING} data with six different 
methods as indicated on the panels; SHCm is the same as SHC20.}
    \label{bowl}
    \end{figure}

 \begin{table}[h!]
 \centering\scriptsize
    \caption{The comparison of the proposed methods on {\sf HALF-RING} data}
 \begin{tabular}{lccccccc} \hline
 &\multicolumn{7}{c}{method}
 \\ \cline{2-8} 
           & K-m  &   CT   &{\sf CHA}      &   SPECC     &   EAC     &   SHCm&SHC20\\ \hline
 MAI   & 0.78      & 0.78  &1    &   1      &   1    &0.99 &0.97 \\
 SAI    & 0      &  0   &  0&0     &   0   &0.044& 0.063 \\
 \hline
 \end{tabular}\\  \label{tabhalf} 
 \end{table}

\subsubsection{ {\sf FLAME} data}
\label{flamedata}

Fu and Medico \cite{FuandMedico} developed a fuzzy clustering technique
for DNA micro-array wich they considered on the test
data given in Figure \ref{flame}.  The website \cite{Webref} refers to this as the {\sf FLAME}
data set.  On this data set, {\sf SHCm} and  {\sf SHC20} are seen to be the methods that best
 identify the two clusters.
None of the other five methods do as well; {\sf CHAMELEON}, {\sf EAC}, and {\sf SPECC} fail completely
because they are greatly distorted by the two apparent outliers.  By contrast, {\sf SHCm}
and {\sf SHC20} deal elaborately with outliers to reduce their effect.  K-m and CT do passably well, but put too many points in the upper cluster.

The overall performance of the methods is summarized in Table \ref{tabflame} indicating 
high MAI and relatively low SAI consistent with Fig. \ref{flame}.

\begin{figure}[htp!]  \centering
  \includegraphics[width=.5\textwidth]{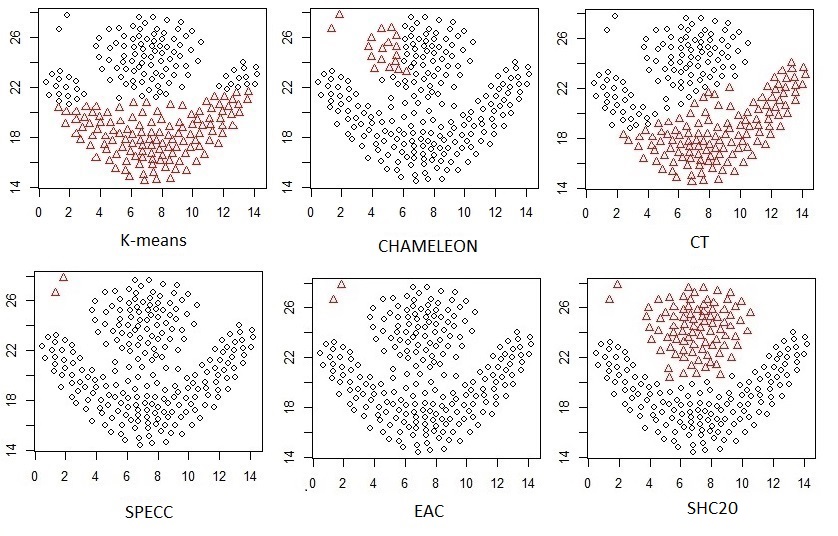} \caption{Clustering of {\sf FLAME} data with six
different methods as indicated on the panels (SHCm and SHC20 are identical).}
  \label{flame}
  \end{figure}
  
\begin{table}[]
\centering \scriptsize
   \caption{The comparison of the proposed methods on {\sf FLAME} data}
\begin{tabular}{lccccccc} \hline
&\multicolumn{7}{c}{method}
\\ \cline{2-8} 
          & K-m  &   CT        & {\sf CHA}&  SPECC     &   EAC     &   SHCm&SHC20\\ \hline
MAI   &       0.84  & 0.84      & 0.71& 0.7         &0.65      &  0.89&0.88\\
SAI    &       0.031&0     &0  &0.122       & 0.0000     & 0.000&0.000 \\
\hline
\end{tabular}\\  \label{tabflame} 
\end{table}

\subsection{Higher dimensional data}
\label{higherdims}

In this context it is difficult to generate two or three 
dimensional plots to evaluate how successful a clustering strategy is visually.
One can use projections onto planes, however, doing this systematically 
increases in difficulty as the dimension
increases.  Hence, we only present tables of MAI and SAI values.     As a generality, {\sf CHA} 
can only be
expected to work well when the clusters are compact and can be separated; this is less and less 
likely as dimension increases and the cases when it does occur tend to be
those where $K$-means performs well and is easier to use.

\subsubsection{ {\sf IRIS} data}

This benchmark data set contains $n=$150 observations with three attributes of Iris flowers. 
Table \ref{tabiris} gives the MAI's and SAI's for the  replications of the methods on the 
{\sf IRIS} data. Obviously, {\sf K-m}, {\sf CT} and the two {\sf SHC}'s methods 
provide more accurate 
clustering than the other methods for this data.   
This is the only one of our examples
where {\sf CT} works well.

\begin{table}[h]
\centering \scriptsize
   \caption{The comparison of the proposed method on the {\sf IRIS} data}
\begin{tabular}{lccccccc} \hline
&\multicolumn{7}{c}{method}
\\ \cline{2-8} 
&K-m&CT&{\sf CHA}&SPECC&EAC&   SHCm&SHC20\\ \hline
MAI&0.89&0.88& 0.73&0.69&0.69&0.89&0.88\\
SAI&0.000&-&-&0.044&0.000&0.054&0.063\\
\hline
\end{tabular}\\  \label{tabiris} 
\end{table}

\subsubsection{ {\sf GARBER} data}

To study the performance of the {\sf SHC}'s with high dimensional data, 
we used the microarray data from \cite{Garberetal}. 
The data are the 916-dimensional gene expression profiles for lung tissue from $n=72$ subjects.  
Of these, five subjects were
normal and  67 had lung tumors.  The classification of the tumors into 6 classes (plus normal)
was done by a pathologist giving seven classes total.  Accordingly, we expect seven clusters.  
%The clustering is based on the classifification of tumor that were done by pathologist, the classification are %adenocarcinoma (AC), squamous cell carcinoma(SCC), large cell lung cancer (LCLC), small cell lung cancer (SCLC) 
%and Patiente were diagnosed with combined LCLC/SCLC where indicated, tumor pairs corresponded to primary %tumor/lymph node (node) and also five normal lung specimens. 
The data set was constructed in \cite{Garberetal} by filling in missing values 
estimated by the means within the same gene profiles. 
Table \ref{tabgarber} presents the results.  
%Clearly, the two versions of {\sf SHC} work best and 
%{\sf EAC} is in second place; this is not a surprise since some steps in our method are
%drawn from {\sf EAC}.   

\begin{table}[h]
\centering \scriptsize
   \caption{The comparison of the proposed method on the {\sf GARBER} data}
\begin{tabular}{lccccccc} \hline
\\&\multicolumn{7}{c}{method}
\\ \cline{2-8} 
data&K-m&CT&{\sf CHA}&SPECC&EAC&SHCm&SHC20\\ \hline
MAI&0.70&0.63&0.54&0.71&0.80&0.82&0.82\\
SAI&0.109&-&-&0.054&0.000&0.000&0.000\\
\hline
\end{tabular}\\  \label{tabgarber} 
\end{table}

\subsubsection{ {\sf WINE QUALITY} data}

The {\sf WINE QUALITY} data is used in Cortez et al. \cite{Cortezetal} to study the  classification of wines.
For red wines, $n=4898$ and seven clusters were found.  For white wine, $n=1599$ and 6 clusters were found.
Both data sets (for red and white wine) have 11 attributes.  Since the data sets are large, we drew $n=300$
observations randomly from each of them.   
Table \ref{tabwine} gives the MAI's and SAI's we found.  In both cases, the best methods
were the two {\sf SHC}'s and {\sf EAC} with the {\sf SHC}'s being slightly better.  
The other three methods were worse
and even {\sf EAC} and the {\sf SHC}'s 
could not be said to perform well except in a relative sense.

 Note that we omitted {\sf CHA} from
this example since it was difficult to use and, as was seen,
{\sf CHA} did poorly
on the first two of these higher dimensional examples
%.  On the other hand, as noted before,
%{\sf K-m} can be expected to perform better than {\sf CHA} when {\sf CHA} works 
%well 
and {\sf K-m} was included.  We also omitted {\sf CT} in this example
since it never did well (except on the {\sf IRIS} data, where three other methods
did as well).
 
 \begin{table}[h]
 \centering
    \caption{The comparison of the proposed method on the {\sf WINE QUALITY} data}
 \begin{tabular}{clcccc} \hline
\multicolumn{6}{c}{red wine} 
\\\hline
&\multicolumn{5}{c}{methods}
 \\ \cline{2-6} 
 data&K-m&SPECC&EAC&   SHCm&SHC20\\ \hline
MAI&0.29& 0.4& 0.48&0.47& 0.48 \\
SAI&0.031&0.070&0.028&0.031&0.031\\
 \hline
\multicolumn{6}{c}{white wine} 
\\\hline
&\multicolumn{5}{c}{methods}
 \\ \cline{2-6} 
 data&K-means&SPECC&EAC&  SHCm&SHC20\\ \hline
  MAI&0.31& 0.38& 0.44& 0.46&0.46\\
  SAI&0.031& 0.0547& 0.002& 0.044&0.031 \\
  \hline
 \end{tabular}\\  \label{tabwine}
 \end{table} 

\subsubsection{Summary of the examples}

From these examples, it is seen that {\sf CT} rarely performs well (only on {\sf IRIS}) and so can be neglected.
{\sf CHA } works well only on two examples, {\sf HALF-RING} and {\sf AGGREGATION}, 
and its performance is
never meaningfully better than both {\sf SHC}'s.  
Likewise, {\sf SPECC} and {\sf EAC} are never meaningfully 
better than both {\sf SHC}'s.    
Finally, {\sf K-m} only performs really well on {\sf 3-NORMALS}, a setting for which
it was designed (and even there has close competitors like {\sf SHC20} and {\sf SPECC}), 
and {\sf IRIS} (although only trivially).  The inference from this is that, 
as a generality, one of {\sf SHCm} and
{\sf SHC20} is the best method.  
%Which of these two methods is to be preferred a priori was
%discussed in Subsec. \ref{SLmerging1}; however, 
The only meaningful difference in our examples
occurred for {\sf AGGREGATION} where {\sf SHC20} outperformed {\sf SHCm}.
Thus, although our theoretical results support {\sf SHCm}, we are led to {\sf SHC20} as a default.

\subsection{Estimating cluster size}
\label{estK}

Estimating the number of clusters is challenging because it requires 
knowing the structure of the underlying population something that is often not known.  That is, identifying the number of groups in a data set
is a problem that is both physically and mathematically challenging. 
Nevertheless, there are several methods to estimate the number of clusters, $\widehat K_{T}$.
For instance, \cite{TibshiraniWaltherandHastie} uses the gap statistic (gap) and it is implemented in the
 \textsf{cluster} package in \textsf{R}.  Another popular method is the Silhouette distance (sil),
 \cite{KaufmanandRousseeuw}, implemented in the \textsf{fpc} package in \textsf{R}.
In addition, one can estimate $K_T$ using the 
Bayesian information criterion (BIC), intializing by a
hierarchical clustering as in implemented, for instance, in {\sf mclust}  in \textsf{R}, see \cite{Frale2007}.

Table \ref{clustersize} shows the estimates of the number of clusters and the
true number of clusters
using six different methods
for the data sets in the previous sub-sections.   The numbers in parentheses are the standard
deviations (SD's) for the estimates.  The bottom row is the absolute error (AE)
formed by taking the sum of the absolute differences between the true and the estimated
number of clusters.
A simple glance at the results shows that if the
raw numbers are rounded to their nearest integers then even though the
{\sf EK} methods based on {\sf SHCm} and {\sf SHC20} do not always identify the correct number of clusters, all other methods perform worse (for the
data sets considered).
Indeed, gap statistic, sil, and BIC do noticably worse.
Only {\sf EAC} is comparable, and it is slightly worse than {\sf EKm}
which is slightly worse than {\sf EK20},
again validating our recommendation of using the 20th percentile, i.e., {\sf EK20}, as a good default.
Unfortunately, the SD's do not seem to provide a helpful guide as to which methods are good;
poor methods can have small SD's and better methods can have larger SD's unlike
for the clustering methods.

 \begin{table}[h]
 \centering
\caption{Estimated numbers of clusters and their SD's using six techniques.}
    {\tiny
 \begin{tabular}{lccccccc} \hline
 &&\multicolumn{6}{c}{methods} 
 \\ \cline{3-8}  data set&actual &EKm&EK20&EAC&Gap&Sil&BIC\\ \hline
 {\sf FLAME} &           2&    2.2(0.3)      &2.1(0.2)       &2.1(0.3)      &2.7(0.8)    &4(0.0)       &4(0.0)\\
{\sf SPIRAL}&             3&    3(0.0)         &3(0.0)          &2.7(1.0)      &7.9(1.4)    &2(0)          &6(0.0)\\
{\sf HALF-RING}&     2&     2.1(0.3)     &2.0( 0.1)      &2.0(0.4)       &1(0)          &20(0)        &20(0.0)\\
{\sf AGGREGATION}&7&     5(0)              &5(0)            &5(0)           &2.3(1)       &4(0)          &10(0)\\
{\sf 3-NORMAL} &    3&     3.4(1.7)      &3.4(1.6)      &3.3(2.3)      &2.4(0.9)     &3.1(0.3)  &3.1(0.4)\\ 
 {\sf IRIS}&                3&     3.0(0.0)      &2.9(0.2)      &2.0(0.3)      &3.4(1.3)     &2(0)          &2(0)\\
 {\sf Red wine}&        6&     3.2(1.6)     &3.2(1.7)       &3.6(3.8)      &8.5(2.8)     &2(0)          &12.0(3.0)\\  
{\sf White wine}&      7&     3.5(2)        &3.8(2.1)        &3.8(2.8)     &10.9(4.4)   &2.4(1.3)  &12.4(4.0)\\  
{\sf GARBER}  &         6&     4.2(1.5)      &4.6(2.2)       &12(10.9)    &5.7(2.5)     &2(0)          &5(0)\\ \hline
\hline
{\sf AE}  &                        &10.8           &  10              &15.3          &      19&35.7&39.5 \\

\hline
 \end{tabular}\\  
\label{clustersize}}
 \end{table} 
 
\section{Conclusion}
\label{conclusions}

%This paper presents a novel approach to robust, accurate estimation of a clustering. 
%Its basic approach is, for any given $K$, to split 
%the data into many (greater than $K$) small clusters, merge them in a way that ensures outliers %or
%other aberrant points have a minimal role in determining
%the clustering, and stabilize the clustering against random noise in
%both the clusters and the data points by using a membership matrix.
%Because the approach is topological, our algorithms are
%effective when the true clustering structure involves
%both convex and non-convex clusters.  As a separate but closely related
%point, our method also gives a reliable technique for
%estimating $K_T$ the true number of clusters.

Our approach leads to two natural techniques that differ in the dis-similarity
used in the single linkage step of our clustering approach.  One dis-similiarity 
is the usual Euclidean distance between two points.  The other is the 20-th percentile
of the distances between points in two different clusters.   
We can establish formal results for the Euclidean distance 
and it has a naural geometric interpretation.  However, the 20-th percentile dissimilarity
gives performance that is no worse and sometimes meaningfully better than
the Euclidean distance.  %Accordingly we have suggested it is the better of the two.

In order to evalaute the performance of the proposed methods, we tested them 
on a wide variety of qualitatively different clusterings, including both real and simulated data,
convex and nonconvex true clusterings, and clusterings
in which the components are not well separated.
Our theory and examples suggest that our methods 
lead to accurate clusterings and that using
$B \approx 250$ to form the membership matrices is enough to provide 
satisfactory stability.
Of course, for more complicated data -- irregular shapes, little separation between
shapes, etc. --  more ensembling 
(higher $B$) may lead to better results.

%One feature of our methods that bears emphasis is that
%we have used single linkage clustering repeatedly as a way
%to combine the clusters from a $K$-means step by merging
%clusters that have data points that may represent the same cluster
%in the data.   
%We have likewise tried to avoid letting single linkage clustering join clusters inappropriately.
%This is reflected in our use of two dissimilarities, Euclidean distance and the 20-th percentile %distance between sets of points, and in our use of a `grow and prune'
%strategy on dendrograms in a later step to deal with influential data points. 

In our examples, our methods effectively equal or outperform many standard
or related
methods such as spectral clustering, $K$-means, {\sf EAC} 
\cite{FredandJain}, hybrid hierachical clustering
\cite{ChipmanandTibshirani}, and {\sf CHAMELEON} \cite{KarypisHanandKumar}.
In fact, in all eight examples we presented here, one of the two
forms of our approach always yielded robust and relatively accurate results.

%The major drawback of our method so far has been its
%computational demands.  Unlike many methods that 
%make one pass over the data, our method makes many passes
%over the data and the computing required to combine
%the results of each iteration can be considerable.
%In our examples, the longest computing time for  {\sf SHC} was for the
%{\sf AGGREGATION} data set which took 38 minutes on a
%MacBook with a 2.3 GHz Intel Core i7.  On the other hand, {\sf SHC} on
%the benchmark data set {\sf IRIS} only took 47 seconds.
%While these may be regarded as high compared to other methods
%(that may have been more efficiently coded), the running time is not
%impossibly high and we consistenly
%get results comparable or better than other methods
%for the range of examples we tested.

%We conclude by noting that although clustering in
%high dimensions is of increasing importance we have
%not focused on this setting.  We note that reducing to a 
%dis-similarity matrix means that it is the sample size that is more important 
%than the dimension of the data.  Separately, it
%is intuitive that as the dimension of a data set
%increases it is typically harder to tell a priori
%whether or not the underlying probability space has convex or non-convex clusters.
%Hence, using a technique such as ours that is
%insensitive to convexity becomes more important as
%dimension increases.  

\section*{Acknowledgment}

The authors gratefully acknowledge support from  the NSF-DTRA, grant No. NR66853W.

\end{document}